\documentclass{article}

\usepackage[preprint]{neurips_2026}
\workshoptitle{Long-Context Foundation Models}

\usepackage{natbib}
\usepackage[utf8]{inputenc} 
\usepackage[T1]{fontenc}    
\usepackage{hyperref}       
\usepackage{url}            
\usepackage{booktabs}       
\usepackage{amsfonts}       
\usepackage{amsthm}         
\usepackage{nicefrac}       
\usepackage{microtype}      
\usepackage{xcolor}         

\usepackage{amsmath}
\usepackage{amssymb}
\usepackage{graphicx}
\newtheorem{proposition}{Proposition}
\newtheorem{lemma}{Lemma}
\newtheorem{corollary}{Corollary}

\title{Underscoring the Problem: Why Softpick Fails at Initialization}

\author{%
  Aryan Sood \\
  IIT Roorkee \\
  \texttt{aryan\_s2@ee.iitr.ac.in}
  \And
  Jaikaran Singh \\
  IIT Roorkee \\
  \texttt{jaikaran\_s@cs.iitr.ac.in}
  \And
  Ishaan Bansal \\
  IIT Roorkee \\
  \texttt{ishaan\_b@ece.iitr.ac.in}
}

\begin{document}

\maketitle

\begin{abstract}
Softmax attention gives every token a nonzero weight, which in trained models concentrates into attention sinks and massive activations that widen the dynamic range low-precision inference must cover. Softpick removes this constraint by rectifying scores, eliminating sinks and lowering hidden-state kurtosis, but its advantage fades at scale, a failure the original work attributes to underscoring and dying heads. We reframe this as a normalization problem. Softpick's denominator splits into positive- and negative-shifted sums $D^+$ and $D^-$, used identically in the forward and backward pass, preventing their roles from being isolated. We separate them into a family of operators that independently choose each denominator. The failure originates at initialization: every layer contains rows where $D^+$ is exactly zero, while near-dead rows produce gradient norms above $10^{12}$ regardless of the backward denominator. Only softpick and a stop-gradient variant, which keeps $D^++D^-$ forward but backpropagates through $D^+$ alone, train from scratch. At 230M parameters the stop-gradient operator matches softpick on quantization, has fewer dead heads, and retrieves passkeys more reliably, trailing only on peak attention-weight kurtosis.
\end{abstract}

\section{Introduction}
\label{sec:introduction}

Large language models are built on attention, where normalization determines how a query distributes weight across the tokens it can see. Since the original transformer \citep{vaswani2017attention}, softmax has been the default, turning raw scores into a distribution that sums to one. That constraint is also its weakness: because every key receives some weight, trained models can concentrate attention on a few content-free tokens, the attention sink \citep{xiao2024efficient,gu2025attention}, and develop massive activations, a handful of hidden-state channels whose magnitude dwarfs the rest of the tensor \citep{sun2024massive}. Both widen the dynamic range that low-precision inference must represent, creating an obstacle for quantized deployment.

Softpick removes this constraint by rectifying a shifted exponential, allowing irrelevant keys to receive an exact zero \citep{zuhri2025softpick}. For scores $x\in\mathbb{R}^N$ and $m=\max_j x_j$,
\begin{equation}
\mathrm{Softpick}(x)_i = \frac{\mathrm{ReLU}(e^{x_i-m}-e^{-m})}{\sum_{j=1}^N \left|e^{x_j-m}-e^{-m}\right| + \epsilon}, \qquad 0 < \epsilon \ll 1.
\label{eq:softpick}
\end{equation}
At small scale this delivers on its promise: sinks disappear, hidden-state kurtosis drops, and low-bit quantization stays close to softmax at comparable training loss. However this does not hold as models grow. \citet{zuhri2025softpick} report that the advantage narrows with scale, a failure they name underscoring: as context lengthens, the negative-shifted part of the denominator grows while the positive part stays roughly constant, diluting attention on relevant keys. Dormant heads are their separate account for the rectified variant.

We instead study the normalization itself. The denominator splits into a positive-shifted sum $D^+$ and a negative-shifted sum $D^-$, so Equation~\ref{eq:softpick}'s denominator is exactly $D^++D^-$ (Section~\ref{sec:problem-formulation}). Softpick uses this sum in both the forward and backward passes, coupling two roles that can fail independently: forward normalization and the gradient through that normalization. We separate these roles by choosing the forward and backward denominator independently, yielding a small family of operators spanning softpick, the fully rectified variant, and three intermediate designs. This decomposition lets us distinguish forward degeneracy from backward effects and identify where the training failure originates. Our contributions are as follows.
\begin{itemize}
\item We define a family of operators over (forward, backward) denominators, and identify the \textbf{stop-gradient operator} ($D^++D^-$ forward, $D^+$ backward) as the only member besides softpick that trains from scratch.
\item We locate the rectified-forward failure at initialization, not training dynamics: layers contain $D^+=0$ rows, while near-dead neighbors produce gradient norms above $10^{12}$, and no floor small enough to preserve rectification restores stability.
\item We compare the stop-gradient operator to softpick at 130M and 230M: most of its cost closes by 230M, where it carries fewer dead heads and retrieves passkeys more reliably, leaving only peak attention-weight kurtosis.
\end{itemize}
\section{Background}
\label{sec:background}

Two concentration phenomena in trained transformers motivate alternatives to softmax attention. \citet{xiao2024efficient} named the attention sink, a small set of tokens that absorb disproportionate attention regardless of content, and \citet{gu2025attention} found the same pattern across most sufficiently trained models. \citet{sun2024massive} identified massive activations, a few hidden-state channels that reach magnitudes far outside the rest of the tensor and force quantization to spend extra precision on outliers \citep{dettmers2022llmint8,xiao2023smoothquant,frantar2023gptq}. \citet{bondarenko2023quantizable} show that these outliers arise partly because attention heads have no cheap way to attend to nothing, while \citet{queipo2026attention} connect sinks and massive activations to the same compression of the residual stream. Softpick \citep{zuhri2025softpick} addresses this cluster with a rectified attention function that needs no sink token and closes most of the gap to softmax at small scale. That gap reopens with scale, and the original work attributes the failure to dying attention heads. We use Softpick as a case study for separating the forward and backward roles of its normalization. A broader survey of attention alternatives is deferred to Appendix~\ref{app:related}.
\section{Problem Formulation}
\label{sec:problem-formulation}

Softpick's denominator plays two roles: it determines the forward normalization of the surviving scores and contributes to the gradient with respect to those scores. Because the same expression controls both, a training failure could arise from the forward normalization, the backward path, or their interaction. We therefore separate these roles.

\paragraph{Splitting the denominator.} Every term $e^{x_i}-1$ keeps the sign of $x_i$, and the max-shift rescaling for numerical stability preserves it (Appendix~\ref{app:derivation}), so each key falls into one of two groups:
\begin{equation}
D^+ = \sum_{j\,:\,x_j \ge 0} \mathrm{ReLU}\!\left(e^{x_j-m}-e^{-m}\right), \qquad
D^- = \sum_{j\,:\,x_j < 0} \left|e^{x_j-m}-e^{-m}\right|,
\label{eq:dplusminus}
\end{equation}
so Equation~\ref{eq:softpick}'s denominator is exactly $D^++D^-$. The numerator depends only on the $D^+$ group, while $D^-$ enters through the denominator and therefore controls the scale of the surviving weights.

\paragraph{The operator family.} Once $D^+$ and $D^-$ are separated, the forward and backward denominators need not be identical. We set the forward denominator $D_f$ and backward denominator $D_b$ independently using a straight-through construction \citep{bengio2013estimating}: the forward value uses $D_f$, while its gradient is computed as though the denominator were $D_b$ (Appendix~\ref{app:implementation}). Table~\ref{tab:operators} lists the five members. Softpick couples $D_f=D_b=D^++D^-$, while the \emph{stop-gradient} operator keeps $D_f=D^++D^-$ but sets $D_b=D^+$, removing the gradient contribution through $D^-$. The remaining operators vary the forward denominator or interpolate between the two components.

\begin{table}[t]
\centering
\small
\caption{The operator family. Each row sets a forward and a backward denominator independently.}
\label{tab:operators}
\begin{tabular}{lllc}
\toprule
Operator & Forward $D_f$ & Backward $D_b$ & Role \\
\midrule
Softpick & $D^++D^-$ & $D^++D^-$ & Coupled baseline \\
Stop-gradient & $D^++D^-$ & $D^+$ (stop-grad on $D^-$) & The survivor \\
Rectified-forward & $D^+$ & $D^++D^-$ & Honest gradient, rectified forward \\
Partial ($\lambda$) & $D^++\lambda D^-$ & matched & $\lambda\in\{0.11,0.3\}$ \\
Floored & $\max(D^+,\delta)$ & $D^++D^-$ & Hard forward floor \\
\bottomrule
\end{tabular}
\end{table}

Inside the forward pass, $D^-$ has two distinct effects: it sets the score scale relevant to long-context retrieval, since a large $D^-$ reduces the weight assigned to positively-scoring keys as negative keys accumulate (underscoring), and it prevents the denominator from approaching zero when $D^+$ is small, since $D^->0$ whenever any key scores negatively. Section~\ref{sec:experiments} tests the first effect by inference-time denominator swaps and the second by training the $D^+$-forward operators from scratch.

\paragraph{A structural singularity at initialization.} A query row is \emph{zero-support} if every key has $x_j<0$, giving $D^+=0$. Softpick's numerator is then exactly zero for every key, and so is its Jacobian, regardless of the backward denominator (Appendix~\ref{app:zerosupport}, Lemma~\ref{lem:dead-row}). Thus the dead-Jacobian condition is a row-level property rather than a head-level one: a head becomes progressively inactive as more of its rows become zero-support, while an exactly zero-support row cannot be revived by changing the backward denominator because its numerator has no gradient.

\begin{proposition}[Zero-support rows at initialization]
\label{prop:zero-support}
At initialization, attention scores are centered with no bias toward positive values row by row, so zero-support rows are expected in every layer, not only at the query's first position. Section~\ref{sec:experiments} confirms this: across five sampled batches, every one of 12 layers contains an exact $D^+=0$ row, with gradient norms between $10^{12}$ and $10^{14}$.
\end{proposition}

On such a row, a $D^+$-only forward denominator gives $0/\epsilon$ and a vanishing Jacobian, while nearby rows can diverge with gradients controlled by $\epsilon$ rather than the data (Appendix~\ref{app:zerosupport}). We therefore test whether removing $D^-$ from the forward denominator is sufficient for stable training, or whether the resulting singularity prevents it.
\section{Experiments}
\label{sec:experiments}

We test the decomposition in three steps: whether $D^-$ has a forward role, whether removing it during training is reachable from scratch, and what cost the stable stop-gradient operator incurs.

\begin{figure}[b]
\centering
\begin{minipage}[b]{0.62\linewidth}\centering
\includegraphics[width=\linewidth]{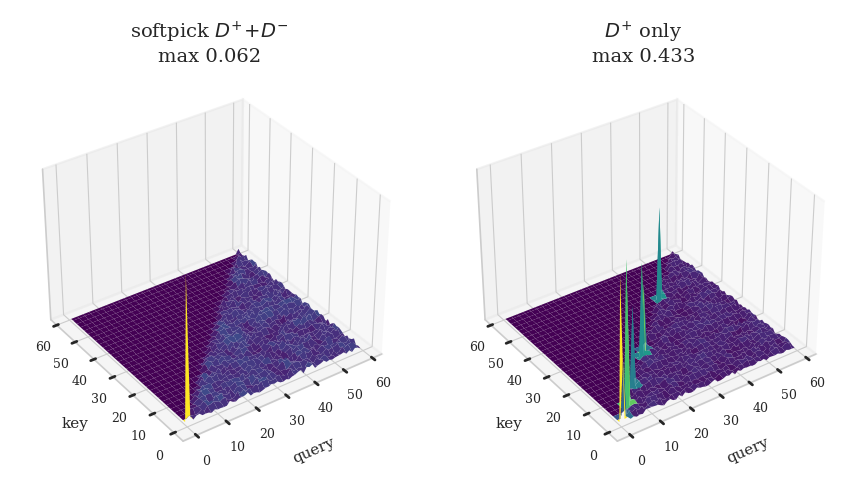}\\[-2pt]
{\footnotesize (a) attention surface, layer 8}
\end{minipage}\hfill
\begin{minipage}[b]{0.36\linewidth}\centering
\includegraphics[width=\linewidth]{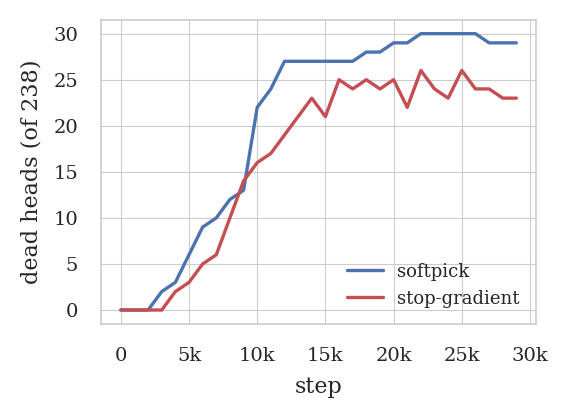}\\[-2pt]
{\footnotesize (b) dead heads, 230M}
\end{minipage}
\caption{(a) Softpick keeps one attention spike per query row (max 0.0625); the $D^+$-only forward scatters spurious spikes at a higher peak (0.4333). (b) The stop-gradient operator ends 230M training with 23 dead heads of 238 against softpick's 29.}
\label{fig:attn-surface}
\label{fig:dead-heads}
\end{figure}

\subsection{A traversable but unreachable interior}
\label{sec:exp-anneal}

Removing $D^-$ from a trained softpick model at inference breaks retrieval: it scatters attention onto the wrong keys rather than merely lowering confidence (Figure~\ref{fig:attn-surface}(a)), and rescaling the score magnitude does not repair the distribution shift (Appendix~\ref{app:inference}). We therefore anneal $D^-$ rather than remove it abruptly. Continuing training of the 340M checkpoint while annealing $c$ in $D^++cD^-$ from $1$ to $0$ over 1000 steps, with matched backward, reveals an interior operating point (Appendix~\ref{app:supp-figures}). Loss and gradient norm remain stable until $c$ falls to about 0.2, then diverge within roughly 100 steps as $c$ approaches zero. Retrieval survives at step-700 ($c=0.111$), even though the 5th percentile of $D^+$ is near zero in all 24 layers and 61 of 384 heads are already dead. Thus the remaining $D^-$ provides a denominator floor while the inherited $c=1$ checkpoint avoids the initialization degeneracy; at $c=0$, retrieval collapses.

From random initialization at 130M, the rectified-forward operator does not learn: loss stays at the random-guess 10.39 while aggregate gradient norms reach $2.9\times10^{12}$ to $5.0\times10^{14}$, and initialization confirms Proposition~\ref{prop:zero-support} (every one of 12 layers has a $D^+=0$ row beyond the first token). Raising $\epsilon$ six orders of magnitude only rescales the singularity, since zero-support rows continue to appear (Appendix~\ref{app:zerosupport}). The partial operator diverges to NaN before 1000 steps at $\lambda=0.11$ and $0.3$, in the same region the anneal crosses only from a stable checkpoint. Only softpick and the stop-gradient operator, both retaining the full $D^++D^-$ forward, train stably from scratch, past 20{,}000 steps at 130M and the full 30{,}000 at 230M.

\subsection{What the stop-gradient costs}
\label{sec:exp-cost}

Softpick and the stop-gradient operator share a forward pass and differ only in whether the backward pass sees $D^-$. At 130M the stop-gradient operator is worse on every measured axis: mean hidden-state kurtosis is 345.8 against 229.1, and its quantization deltas are higher. By 230M most of that gap closes, with hidden-state kurtosis within one percent and indistinguishable 4-bit RTN results; only peak attention-weight kurtosis remains roughly four times worse (Appendix~\ref{app:full-cost}). Dead heads move in its favor: it ends at 23 of 238 against softpick's 29 (Figure~\ref{fig:dead-heads}(b)), replacing the earlier categorical difference with a smaller but consistent reduction. Passkey retrieval separates the two most sharply. At 230M the stop-gradient operator retrieves verbatim at $\sim$1024 and $\sim$2048 tokens (100\% and 77.6\%), while softpick scores zero. Per-sample digit-level analysis of softpick's 500 generations per length confirms a genuine failure, not a scoring artifact: it emits the target digit string in $0\%$ of cases, its first generated number is never correct, and most generations contain no digits (Appendix~\ref{app:scorer}). At 130M, however, softpick retrieves well (99.8\% at $\sim$1024), so the reversal occurs between 130M and 230M, consistent with underscoring. One checkpoint per cell cannot separate this from run variance. Both collapse past their training context.

\subsection{Why the other operators fail}
\label{sec:exp-ablation}

The remaining operators fail with a small forward denominator that amplifies the gradient it routes (Appendix~\ref{app:family}). The floored operator removes the initialization singularity and begins to learn, yet never stabilizes: its 230M loss oscillates around 4 nats against the stable operators' 2.75, and its gradient norm spikes to $6\times10^{9}$. No tested floor preserving rectification escapes this amplification. The rectified-forward operator carries the full gradient and still does not train. Every failing member pairs a small $D_f$ with a $D^-$-inclusive backward, leaving forward degeneracy and backward coupling entangled. The stop-gradient operator is the only tested variant that retains the full forward denominator while training from scratch.
\section{Conclusion}
\label{sec:conclusion}

Decomposing Softpick's denominator into independent forward and backward components reveals a structural initialization failure in rectified-forward variants (Proposition~\ref{prop:zero-support}). Among the tested operators, only the stop-gradient variant, which retains $D^++D^-$ in the forward pass while backpropagating through $D^+$, trains stably from scratch. Its cost largely closes by 230M, where it carries fewer dead heads and retrieves passkeys where Softpick fails, leaving peak attention-weight kurtosis as its main measured disadvantage (Appendix~\ref{app:future}). These results motivate broader study of decoupled forward and backward normalization beyond Softpick.

\bibliographystyle{plainnat}

\bibliography{references}

\newpage
\appendix

\section{Extended related work}
\label{app:related}

Alternatives to softmax attention are a long line of work. Sparsemax \citep{martins2016from} replaces the exponential normalization with a Euclidean projection onto the probability simplex and can assign exact zero probability to low-scoring keys; $\alpha$-entmax \citep{peters2019sparse} generalizes this into a family interpolating between softmax and sparsemax, later given a learnable per-head sparsity \citep{correia2019adaptively} and a continuous form \citep{martins2020sparse}. These methods keep the sum-to-one constraint but allow zeros within it, whereas softpick and the operators studied here drop the constraint entirely, so their rows need not normalize to one. A separate line restricts which keys a query may attend to for efficiency rather than for sparsity of the weights, using fixed or strided patterns \citep{child2019generating}. Non-probabilistic replacements drop the normalization framing outright: ReLU attention \citep{wortsman2023replacing} and sigmoid attention \citep{ramapuram2025sigmoid} both remove the coupling across keys that softmax enforces. Softpick sits between these, rectifying like ReLU attention but retaining a denominator, which is precisely the object this paper decomposes. On the deployment side, the outlier and quantization literature \citep{dettmers2022llmint8,xiao2023smoothquant,frantar2023gptq,bondarenko2023quantizable} is what makes the activation-statistics axis (hidden-state kurtosis, low-bit loss deltas) a meaningful cost to track rather than a curiosity.

\section{Implementation details}
\label{app:implementation}

This section documents the code the experiments run on, so the operator family of Table~\ref{tab:operators} can be checked against a real implementation rather than only against Equation~\ref{eq:softpick}.

\paragraph{Base library and checkpoint.} All models are instances of \texttt{TransformerForCausalLM} from the \texttt{fla} (flash-linear-attention) library, with the now-standard LLaMA-style stack \citep{touvron2023llama}: RMSNorm \citep{zhang2019root}, rotary position embeddings \citep{su2021roformer}, and a gated SwiGLU MLP \citep{shazeer2020glu}. The 340M checkpoint used for the inference swaps and the anneal (Appendix~\ref{app:inference}, Section~\ref{sec:exp-anneal}) is the publicly released softpick 340M model \citep{zuhri2025softpick}: hidden size 1024, 24 layers, 16 heads (384 attention heads in total, the denominator behind every dead-head fraction we quote for this checkpoint), vocabulary 32{,}000, RoPE base 10{,}000, norm $\epsilon=10^{-6}$.

\paragraph{The straight-through decoupling.} The whole operator family is realized by one construction. The attention weights are $\mathrm{ReLU}(\cdot)/\tilde D$ with
\[
\tilde D = D_f.\mathrm{detach}() + D_b - D_b.\mathrm{detach}() + \epsilon,
\]
whose forward value is exactly $D_f$ (the middle two terms cancel numerically) while its gradient is taken as though the denominator were $D_b$. Setting $D_f=D_b=D^++D^-$ recovers softpick. The stop-gradient operator uses $D_f=D^++D^-$ and $D_b=D^+$; the rectified-forward operator uses $D_f=D^++\lambda D^-$ with $\lambda=0$ and $D_b=D^++D^-$; the floored operator uses $D_f=\max(D^+,\delta)$. This is the exact code path behind the from-scratch operator-grid runs, and it is where the $D^+$/$D^-$ split of Section~\ref{sec:problem-formulation} lives in practice.

\paragraph{Fast kernels, and an asymmetry worth noting.} We implement fast Triton kernels for the three nontrivial forward-modified operators, building on the library's parallel softpick attention kernel. The stop-gradient kernel changes only the backward sign term, from \texttt{where(s>0, delta, -delta)} to \texttt{where(s>0, delta, 0.0)}, which is precisely the removal of $D^-$'s backward contribution. The rectified-forward and floored kernels instead change the forward denominator and reuse softpick's backward kernels unchanged, because those kernels never recompute the forward denominator: they consume a saved log-sum-exp tensor, so writing the modified $D_f$ into that tensor makes the existing backward pick up the right scaling. The rectified-forward operator additionally needs a hand-written \texttt{torch.autograd.Function} so that its $D^+$-only forward can still hand back the full $D^++D^-$ gradient, since ordinary autograd through a forward that never used $D^-$ could not produce that gradient on its own. The asymmetry is worth stating: the operator that survives, the stop-gradient one, is the cheapest change to the backward pass, while the operator that fails is the one that needs custom backward machinery to exist at all.

\paragraph{Two values of $\epsilon$.} The from-scratch operator-grid runs and the initialization diagnostic use $\epsilon=10^{-6}$, matching the softpick paper's stated value and our training default. The 340M anneal continuation uses $\epsilon=10^{-8}$, the value the reference implementation actually uses, chosen there for consistency with every other experiment run against that particular checkpoint. The difference is immaterial to the conclusions but is recorded here so the reported gradient magnitudes can be reproduced exactly.

\paragraph{Passkey harness mechanics.} We use the landmark/passkey retrieval task \citep{mohtashami2023landmark}. Each prompt places a passkey integer sampled from $[1,50000)$ at a random offset inside \texttt{n\_garbage} $=\lceil 3.6\times\text{ctx\_len}\rceil$ characters of repeated filler, and asks the model to complete ``The pass key is''. The harness then calls
\begin{verbatim}
L = answer_ids.shape[-1]
gen = model.generate(..., max_new_tokens=L, do_sample=False)
is_correct = (gen[0, -L:] == answer_ids[0]).all().item()
\end{verbatim}
The model is allowed to generate exactly as many tokens as the ground-truth passkey tokenizes to, and no more. In principle this fixed budget could truncate or misalign a correct-but-differently-formatted answer before scoring, which would understate exact match. We test for exactly this on softpick's 230M zero with the per-sample digit-level analysis of Appendix~\ref{app:scorer}, using generations that run well past the passkey length, and find its zero is a genuine retrieval failure rather than a budget artifact.

\section{Numerically stable softpick and why it does not fix zero-support rows}
\label{app:derivation}

The unshifted softpick expression, $\mathrm{ReLU}(e^{x_i}-1)$ over $\sum_j |e^{x_j}-1|+\epsilon$, overflows for large $x_i$. The stable form in Equation~\ref{eq:softpick} multiplies numerator and denominator by $e^{-m}$, where $m=\max_j x_j$:
\[
e^{-m}\left(e^{x_i}-1\right) = e^{x_i-m} - e^{-m}.
\]
This rescaling changes the size of every term by the same constant factor, and it changes neither the sign of any term nor their ratios. A row's zero-support status, whether every key has $x_j<0$, is therefore identical under the stable and unstable forms. Adding $\epsilon$ to the full denominator $D^++D^-$ only guards against $D^++D^-$ itself being zero, which needs every key in the row to score exactly zero and is measure-zero at a continuous initialization. The far more common failure, $D^+=0$ with $D^->0$, is untouched by $\epsilon$ unless $\epsilon$ is added directly to a $D^+$-only forward denominator, which is why raising $\epsilon$ helps the rectified-forward operator only briefly (Section~\ref{sec:exp-anneal}).

\section{Zero-support rows: lemma and gradient analysis}
\label{app:zerosupport}

Fix one query row with score vector $x\in\mathbb{R}^N$, let $m=\max_j x_j$, write $s_i = e^{x_i-m}-e^{-m}$, and let the attention weight be $a_i = \mathrm{ReLU}(s_i)/(D+\epsilon)$ with $D=\sum_j |s_j|$. Recall $D^+ = \sum_{j:x_j\ge 0}\mathrm{ReLU}(s_j)$ and $D^-=\sum_{j:x_j<0}|s_j|$, so $D=D^++D^-$. A row is \emph{zero-support} if $x_j<0$ for every $j$, equivalently $D^+=0$.

\begin{lemma}[Dead row]
\label{lem:dead-row}
On a zero-support row, $a_i=0$ for every $i$, and the Jacobian $\partial a_i/\partial x_k = 0$ for all $i,k$, independently of which denominator is used for the backward pass.
\end{lemma}

\begin{proof}
Since every $x_j<0$, the maximum satisfies $m<0$, hence $e^{-m}>1$. For each $i$, $x_i\le m$ gives $e^{x_i-m}\le 1<e^{-m}$, so $s_i=e^{x_i-m}-e^{-m}<0$ strictly. Therefore $\mathrm{ReLU}(s_i)=0$ for all $i$, and $a_i=0$. Because the inequality $s_i<0$ is strict, it holds on an open neighborhood of $x$, on which $\mathrm{ReLU}(s_i)\equiv 0$ as a function of the scores. The numerator of $a_i$ is thus locally the constant zero function, so $\partial\,\mathrm{ReLU}(s_i)/\partial x_k=0$; with a zero numerator and a nonzero denominator $D+\epsilon>0$, the quotient rule gives $\partial a_i/\partial x_k=0$. No factor in this argument depends on the backward denominator $D_b$: the score-dependent part of $a_i$ that could carry gradient is the numerator, and it is locally constant.
\end{proof}

Lemma~\ref{lem:dead-row} is the row-level form of the original dead-Jacobian argument: a dead row cannot be revived by any choice of backward denominator, because there is no gradient in its numerator to route. A head dies only as an aggregate of its rows turning zero-support, which is why head death is gradual while row death is immediate.

\paragraph{Gradient blow-up near the dead set.} The lemma covers the exactly-dead row. The instability of the rectified-forward operator comes from \emph{near}-dead rows under its $D^+$-only forward. Consider a row whose single positive key is barely positive, so $D^+=\eta\to 0^+$; taking that key as the arg max, its numerator is $p=\mathrm{ReLU}(1-e^{-x})\approx x=\eta$ with $p'=\partial p/\partial x\approx 1$. Its forward weight is $a=p/(\eta+\epsilon)$, and
\[
\frac{\partial a}{\partial x}
= \frac{p'(\eta+\epsilon) - p\,\eta'}{(\eta+\epsilon)^2}
= \frac{(\eta+\epsilon) - \eta}{(\eta+\epsilon)^2} + O\!\Big(\tfrac{\eta}{(\eta+\epsilon)^2}\Big)
= \frac{\epsilon}{(\eta+\epsilon)^2} + O\!\Big(\tfrac{\eta}{(\eta+\epsilon)^2}\Big),
\]
so $\|\partial a/\partial x\| = \Theta\big((\eta+\epsilon)^{-1}\big)$. This gives the following boundary behavior.

\begin{corollary}[Discontinuity across the dead set]
\label{cor:discontinuity}
For the $D^+$-only forward, $\|\partial a/\partial x\|$ grows as $\eta^{-1}$ as $\eta\to 0^+$ away from the $\epsilon$ floor, saturates at $\Theta(\epsilon^{-1})$ once $\eta\lesssim\epsilon$, and then drops to exactly $0$ on the dead set $\eta=0$ (Lemma~\ref{lem:dead-row}). The per-row gradient is therefore discontinuous across the zero-support boundary.
\end{corollary}

This discontinuity is why the operator is unstable rather than merely slow: an optimizer that walks a row toward the dead set sees exploding gradients right up to the boundary. It also predicts the floored operator's behavior exactly: flooring $D^+$ at $\delta$ pushes rows off the exact-zero set but leaves them in the $\eta^{-1}$ regime with $\eta\ge\delta$, so the model learns yet never stabilizes (Section~\ref{sec:exp-ablation}). The per-row constant $\epsilon^{-1}$ is $10^{6}$ at $\epsilon=10^{-6}$; the larger norms observed in practice, $10^{12}$ to $10^{14}$ (Section~\ref{sec:exp-anneal}), reflect this per-row blow-up propagated through the score-to-parameter Jacobian and summed over the many near-dead rows per layer (below), and we characterize the per-row scaling rather than claim a closed form for the aggregate. Raising $\epsilon$ rescales the constant once, but each optimizer step produces fresh near-dead rows, so the blow-up recurs: the $\epsilon=1.0$ smoke test climbs from $\approx 3$ back past $100$ within 60 steps and to $397.6$ by step 99.

\paragraph{Expected number of zero-support rows at initialization.} We make Proposition~\ref{prop:zero-support} quantitative under an explicit idealization.

\begin{proposition}[Expected zero-support count]
\label{prop:count}
Model the key scores of a row at initialization as independent and symmetric about zero. This is the natural first-order model at init: with $Q,K$ drawn from independent zero-mean Gaussian weights acting on normalized inputs, each score $x_j=\langle q,k_j\rangle/\sqrt{d}$ is mean-zero and marginally symmetric, so $\mathbb{P}(x_j<0)=\tfrac12$. For a causal query at position $t$ (with $t$ visible keys), $\mathbb{P}(\text{zero-support})=\mathbb{P}(\text{all }t\text{ keys}<0)=2^{-t}$. Summing over the $L$ query positions of a length-$L$ sequence, the expected number of zero-support rows per head is
\[
\sum_{t=1}^{L} 2^{-t} = 1 - 2^{-L} < 1,
\]
and over $H$ heads per layer the expectation is $H\,(1-2^{-L})\approx H$, concentrated at the earliest positions ($t{=}1$ contributes $\tfrac12$, $t{=}2$ contributes $\tfrac14$).
\end{proposition}

So the expected count is $\approx H$ zero-support rows per layer, not merely ``at least one,'' and they sit at low query positions where few keys are visible, exactly the empirical pattern of Section~\ref{sec:exp-anneal}. The idealization is only that: keys share the $K$ projection and RoPE correlates neighbors, so scores are not independent, and any additive score bias breaks symmetry. We therefore read Proposition~\ref{prop:count} as a computed expectation under a named, defensible model whose prediction ($\Theta(H)$ per layer, early positions) matches the measured count ($\ge 1$ in every one of 12 layers, at positions beyond the first), rather than as an exact law.

\section{Why the stop-gradient operator is stable, and the bias it pays}
\label{app:stability}

Corollary~\ref{cor:discontinuity} shows the $D^+$-only forward blows up near the dead set. The surviving operators avoid this for a reason that is equally provable.

\begin{lemma}[Bounded gradient of the full-denominator forward]
\label{lem:bounded}
For softpick and the stop-gradient operator, whose forward denominator is $D_f=D^++D^-=\sum_j|s_j|$, we have $D_f=0$ only if $x_j=0$ for every $j$, a measure-zero event at a continuous initialization. Hence $D_f>0$ almost surely, and on any row with at least one strictly negative key $x_{\min}<0$,
\[
D_f \;\ge\; D^- \;\ge\; |s_{\min}| \;=\; e^{-m}\big(1-e^{x_{\min}}\big) \;>\; 0,
\]
a bound that does not involve $\epsilon$. The resulting weight gradient is $\|\partial a/\partial x\|=O(1/D_f)$, almost surely finite; for the stop-gradient operator the backward routes through $D^+$ but its two terms both carry the forward denominator, $\partial a/\partial x = \mathrm{num}'/(D^++D^-) - \mathrm{num}\,(D^+)'/(D^++D^-)^2 = O\big(1/(D^++D^-)\big)$, so it is $O(1/D_f)$ rather than $O(1/\epsilon)$.
\end{lemma}

This is the formal complement of Corollary~\ref{cor:discontinuity}: the failing operator has $\mathbb{P}(D_f=0)=\Theta(1)$ per row (Proposition~\ref{prop:count}) and an $\epsilon$-controlled gradient, while the surviving operators have $\mathbb{P}(D_f=0)=0$ and a gradient bounded by the score configuration alone. The $\epsilon$ floor is never load-bearing for them.

\paragraph{Why the annealed interior exists.} The same bound explains Section~\ref{sec:exp-anneal}'s ledge. Along the anneal, the forward denominator is $D^++cD^-\ge cD^->0$ on any row with a negative key, so for every $c>0$ the forward is non-degenerate and the singularity appears only exactly at $c=0$. The interior point $c=0.111$ is therefore traversable in principle. The from-scratch failure of the partial operator is not that the forward is singular but that random initialization places $\approx H$ rows per layer in the near-dead regime (Proposition~\ref{prop:count}), where a small $c$ makes the $cD^-$ floor weak; the anneal escapes this only because it inherits weights from the trained $c=1$ model, whose rows already sit away from the dead set. This states the mechanism precisely without claiming a basin-of-attraction proof, which is out of scope here.

\paragraph{The bias the stop-gradient pays.} The stop-gradient operator evaluates with $D^++D^-$ but differentiates as though the denominator were $D^+$, so its update direction is not the gradient of the objective it evaluates. It is a biased estimator of that gradient, differing from it by the omitted term that flows through $D^-$. This bias is inherent to the straight-through construction \citep{bengio2013estimating} and we do not characterize it. Empirically the operator trains stably to a loss comparable to softpick (Appendix~\ref{app:full-cost}); why a biased-but-stable update converges to a sensible optimum, and how large the bias is, we leave to future work. We name it here rather than leave it implicit, since it is the honest mathematical status of the operator.

\section{Passkey scoring}
\label{app:scorer}

Our headline metric is exact match: the generated answer must equal the target digit string within the fixed output-token budget of Appendix~\ref{app:implementation}. Exact match can in principle be format-fragile, scoring zero for an answer that is correct but longer or differently tokenized, so we checked whether softpick's 230M zero (Section~\ref{sec:exp-cost}) is such an artifact by computing digit-level metrics on all $n=500$ generations per context length (Table~\ref{tab:digit-level}). It is not an artifact. Softpick emits the target digit string in $0\%$ of cases at every context length, its first generated number matches the target in $0\%$ of cases, and most of its generations contain no digits at all (only 32\% / 43\% / 0\% do). The 23 to 33\% digit-level recall at short context is longest-common-subsequence overlap between two short random numbers, no better than chance, and it falls to $0\%$ at $\sim$3200 where the model emits no digits. So softpick genuinely fails to retrieve at 230M rather than retrieving in a form the scorer misses; the stop-gradient operator, by contrast, emits the target verbatim (its exact-match score is its digit-level score). This also supersedes an earlier extended-budget substring number for softpick ($29.2\%$ / $41.2\%$): with a long enough generation the target digits appear by coincidence, which the per-sample analysis rules out at $n=500$. Quantization is reported at 2-bit and 4-bit round-to-nearest only.

\begin{table}[h]
\centering
\small
\caption{Digit-level analysis of softpick's 230M generations, $n=500$ per context length. All strict digit metrics are $0\%$; the low recall is chance overlap, and most generations contain no digits.}
\label{tab:digit-level}
\begin{tabular}{lccc}
\toprule
Metric & ctx $\sim$1024 & ctx $\sim$2048 & ctx $\sim$3200 \\
\midrule
Exact match & 0\% & 0\% & 0\% \\
Target digits appear (substring) & 0\% & 0\% & 0\% \\
First generated number correct & 0\% & 0\% & 0\% \\
Digit recall (LCS / target length) & 23.4\% & 32.7\% & 0.0\% \\
Mean normalized edit distance & 0.77 & 0.67 & 1.00 \\
Generations containing any digit & 32.0\% & 43.0\% & 0.0\% \\
\bottomrule
\end{tabular}
\end{table}

\section{Inference-time denominator swaps}
\label{app:inference}

This is the inference experiment referred to at the start of Section~\ref{sec:exp-anneal}. If $D^-$ only sets the size of the attention weights, rescaling a $D^+$-only output to the right magnitude should recover softpick's behavior. We take a softpick-trained 340M checkpoint, swap in different forward denominators at inference only (no training), and measure passkey retrieval (full sweep in Table~\ref{tab:inference-full}). Softpick retrieves at 96.7\% / 93.3\% / 86.7\% across $\sim$1024 / $\sim$2048 / $\sim$3200 tokens. Forwarding on $D^+$ alone gives gibberish and 0\% retrieval at every length, whether or not the output is rescaled to match the original magnitude, and mixing in a fraction $c$ of $D^-$ shows a sharp threshold rather than a smooth recovery: $c=0.10$ still collapses to zero, while $c=0.30$ jumps back near baseline. Figure~\ref{fig:attn-surface}(a) shows why rescaling fails. Softpick's attention surface is sparse, one clean spike per query row; removing $D^-$ does not soften that spike toward uniform attention but scatters new spikes across unrelated key positions, at a peak weight exceeding the original signal. The failure is the wrong keys winning, not underconfidence, so the score-scale and non-degeneracy roles of $D^-$ cannot be separated by a static swap, which motivates the annealing experiment of Section~\ref{sec:exp-anneal}. The layer-2 companion surface is Figure~\ref{fig:attn-surface-layer2}.

\section{Full inference-time and annealing sweeps}
\label{app:full-tables}

Table~\ref{tab:inference-full} is the full inference-only denominator sweep of Appendix~\ref{app:inference}, including the length-scaled variants.

\begin{table}[h]
\centering
\small
\caption{Full inference-only denominator sweep, 340M checkpoint, $n=30$ per cell.}
\label{tab:inference-full}
\begin{tabular}{lccc}
\toprule
Condition & ctx $\sim$1024 & ctx $\sim$2048 & ctx $\sim$3200 \\
\midrule
abs (softpick baseline) & 96.7\% & 93.3\% & 86.7\% \\
relu ($D^+$ only) & 0\% & 0\% & 0\% \\
relu, rescaled & 0\% & 0\% & 0\% \\
damped $c=0.01$ & 0\% & 0\% & 0\% \\
damped $c=0.05$ & 0\% & 0\% & 0\% \\
damped $c=0.10$ & 0\% & 0\% & 0\% \\
damped $c=0.30$ & 100\% & 86.7\% & 66.7\% \\
damped $c=0.50$ & 100\% & 96.7\% & 66.7\% \\
damped $c=0.70$ & 96.7\% & 93.3\% & 76.7\% \\
damped $c=1.00$ (= abs) & 96.7\% & 93.3\% & 86.7\% \\
length-scaled $\gamma=0.25$ & 90\% & 0\% & 0\% \\
length-scaled $\gamma=0.50$ & 0\% & 0\% & 0\% \\
length-scaled $\gamma=0.75$ & 0\% & 0\% & 0\% \\
length-scaled $\gamma=1.00$ & 0\% & 0\% & 0\% \\
\bottomrule
\end{tabular}
\end{table}

The damped sweep shows a threshold on both sides of $c=0.30$: retrieval is at baseline for $c\ge0.30$ and exactly zero for $c\le0.10$, with no tested condition between. Length-scaling, a per-token rescaling by an estimate of context length rather than a fixed coefficient, fails everywhere except a partial, short-context recovery at $\gamma=0.25$, which reinforces that the failure is not about output magnitude.

Table~\ref{tab:anneal-full} is the full checkpoint sweep for the $c{:}1\to0$ anneal. The schedule is analytic, not read off logs: with 10 stairs of 100 steps, $c$ at stair $k$ is $(1-k/9)^2$ for $k=\min(\lfloor\text{step}/100\rfloor,9)$, so step-500 sits at $c_4=0.3086$, step-700 at $c_6=0.1111$, step-900 at $c_8=0.0123$, and step-1000 at $c_9=0$. These match the logged values to three decimals, confirming the schedule ran as specified.

\begin{table}[h]
\centering
\small
\caption{Passkey retrieval along the anneal, 340M. Step-700 is the last checkpoint before collapse.}
\label{tab:anneal-full}
\begin{tabular}{lcccc}
\toprule
Checkpoint & $c$ & ctx $\sim$1024 & ctx $\sim$2048 & ctx $\sim$3200 \\
\midrule
Baseline & 1.000 & 96.7\% & 93.3\% & 86.7\% \\
step-500 & 0.309 & 100.0\% & 53.3\% & 40.0\% \\
step-700 & 0.111 & 96.7\% & 90.0\% & 63.3\% \\
step-900 & 0.012 & 83.3\% & 80.0\% & 30.0\% \\
step-1000 & 0.000 & 0\% & 0\% & 0\% \\
\bottomrule
\end{tabular}
\end{table}

\section{Full kurtosis, quantization, and retrieval tables}
\label{app:full-cost}

Table~\ref{tab:kurtosis-quant-full} gives the complete softpick versus stop-gradient comparison of Section~\ref{sec:exp-cost} at both scales. The two scales use different attention-weight aggregations, mean at 130M and peak single-layer (layer 3) at 230M, since that is what was available at each scale; read that row only within a scale, not across.

\begin{table}[h]
\centering
\small
\caption{Softpick vs.\ stop-gradient operator, full matched-scale comparison. Single checkpoint per operator per scale, no seeds.}
\label{tab:kurtosis-quant-full}
\resizebox{\linewidth}{!}{%
\begin{tabular}{lcccc}
\toprule
& \multicolumn{2}{c}{130M} & \multicolumn{2}{c}{230M} \\
Metric & Softpick & Stop-grad & Softpick & Stop-grad \\
\midrule
Baseline loss (fp32) & 2.9647 & 2.9729 & 2.9091 & 2.8569 \\
Mean hidden-state kurtosis & 229.06 & 345.81 & 700 & 695 \\
Attn.-weight kurtosis (mean 130M / peak layer-3 230M) & 36{,}363.89 & 31{,}837.27 & 7{,}894{,}308 & 30{,}128{,}298 \\
2-bit RTN loss delta & +5.81 & +6.16 & +6.32 & +6.52 \\
4-bit RTN loss delta & +2.42 & +4.62 & +6.29 & +6.13 \\
Dead heads at end of run (eps=$10^{-6}$) & 2 & 0 & 29/238 & 23/238 \\
\bottomrule
\end{tabular}}
\end{table}

Table~\ref{tab:passkey-full} gives the full 230M passkey comparison. The softpick row includes two random seed offsets, to check that its near-zero exact-match score is not a seed artifact. It is not.

\begin{table}[h]
\centering
\small
\caption{Passkey retrieval, 230M, exact match, $n{=}500$ per cell unless noted, matched output-length budget.}
\label{tab:passkey-full}
\begin{tabular}{lccc}
\toprule
Condition & ctx $\sim$1024 & ctx $\sim$2048 & ctx $\sim$3200 \\
\midrule
Softpick, seed offset 0 & 0.0\% (0/500) & 0.0\% (0/500) & 0.0\% (0/500) \\
Softpick, seed offset 100{,}000 & 0.0\% (0/500) & 0.2\% (1/500) & 0.0\% (0/500) \\
Stop-gradient, seed offset 0 & 100.0\% (500/500) & 77.6\% (388/500) & 0.0\% (0/500) \\
\bottomrule
\end{tabular}
\end{table}

A digit-level breakdown of softpick's generations (Table~\ref{tab:digit-level}) confirms its zero is genuine: it does not emit the passkey digits, so exact match is not understating it.

At 130M, before softpick's output formatting had degraded, exact-match retrieval instead favors softpick (Table~\ref{tab:passkey-130m}), so the scale reversal noted in Section~\ref{sec:exp-cost} is genuine rather than a scorer artifact. The 130M training losses are also close, 2.694 (softpick) and 2.708 (stop-gradient) at step 30k.

\begin{table}[h]
\centering
\small
\caption{Passkey retrieval, 130M, exact match, $n{=}500$ per cell.}
\label{tab:passkey-130m}
\begin{tabular}{lcc}
\toprule
Context & Softpick & Stop-gradient \\
\midrule
$\sim$1024 & 99.8\% (499/500) & 85.6\% (428/500) \\
$\sim$2048 & 65.8\% (329/500) & 55.2\% (276/500) \\
\bottomrule
\end{tabular}
\end{table}

\section{The full operator-family boundary}
\label{app:family}
This expands Section~\ref{sec:exp-ablation}: the detailed failure of each operator, in order of how much each one narrows down the cause. Table~\ref{tab:family-status} collects the outcomes.
\begin{table}[h]
\centering
\small
\caption{Every operator trained from scratch, and how it failed. Softpick and the stop-gradient operator are the only two that train.}
\label{tab:family-status}
\begin{tabular}{llll}
\toprule
Operator & Forward & Backward & Outcome (from scratch) \\
\midrule
Softpick & $D^++D^-$ & $D^++D^-$ & Trains stably to 28k+ steps \\
Stop-gradient & $D^++D^-$ & $D^+$ & Trains stably to 28k+ steps \\
Rectified-forward & $D^+$ & $D^++D^-$ & Loss flat $\sim$10.39, grad norm $10^{12}$--$10^{14}$ \\
Partial (0.11, 0.3) & $D^++\lambda D^-$ & matched & NaN before 1000 steps \\
Floored & $\max(D^+,\delta)$ & $D^++D^-$ & Learns, oscillates $\sim$4 nats, grad norm spikes \\
Matched-rectified & $D^+$ & $D^+$ & Not run (Section~\ref{sec:family-gap}) \\
\bottomrule
\end{tabular}
\end{table}
\paragraph{The floored operator shows the floor cannot be made small enough.} Replacing the forward denominator with $\max(D^+,\delta)$, a small per-layer floor sampled from the low percentile of healthy $D^+$ rows, while keeping the rectified-forward operator's full backward gradient, removes the initialization singularity: the model learns rather than pinning at the random-guess value. It never stabilizes (Figure~\ref{fig:floored-c}): loss oscillates around 4 to 4.5 nats against the stable operators' 2.75, with spikes past 8, and the gradient norm spikes to nearly $6\times10^{9}$ near step 19{,}000. We stop the 130M run at 18{,}000 steps once NaNs accumulate. Corollary~\ref{cor:discontinuity} explains why: under the straight-through construction the forward denominator enters the weight gradient as $\partial a/\partial x = \mathrm{num}'/D_f - \mathrm{num}\,D_b'/D_f^2$, so flooring at $\delta$ caps the amplification at $\Theta(\delta^{-2})$ rather than removing it. A floor large enough to tame that factor would restore most of $D^-$ and abandon rectification, which is the property the operator exists to provide.
\begin{figure}[h]
\centering
\includegraphics[width=\linewidth]{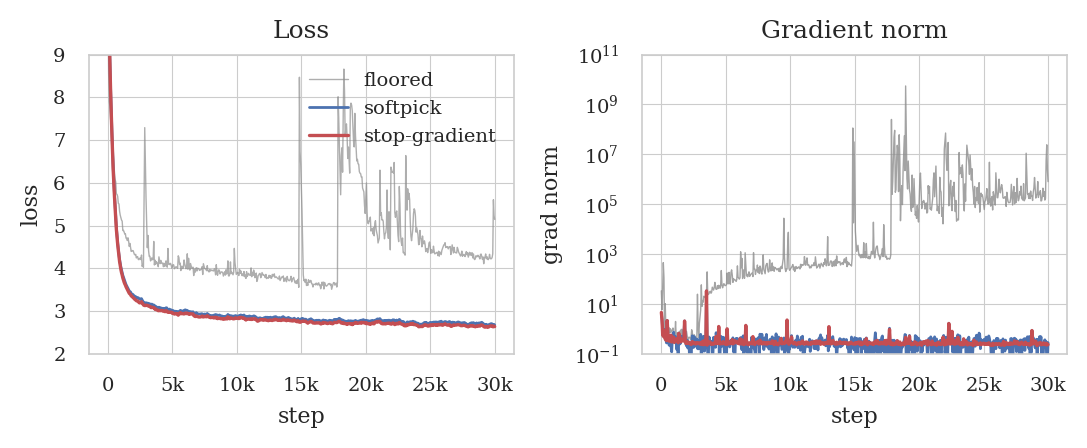}
\caption{Loss (left) and gradient norm (right) over training at 230M, for softpick, the stop-gradient operator, and the floored operator. The floored operator learns instead of pinning at random-guess loss, then oscillates well above the two stable operators and spikes to nearly $6\times10^{9}$ in gradient norm near step 19{,}000.}
\label{fig:floored-c}
\end{figure}
\paragraph{$\epsilon$ rescales the problem rather than fixing it.} Raising the rectified-forward operator's $\epsilon$ to 1.0 lowers its initial gradient norm but does not stop the climb past 100 within 60 steps and past 397 by step 99. New zero-support rows keep appearing, and each reopens the same division no matter where $\epsilon$ is set.
\paragraph{The partial operator fails from scratch even where annealing survives.} The matched operator $D^++\lambda D^-$, trained from random initialization at $\lambda=0.11$ and $0.3$, produces NaN gradients before 1000 steps at both values, in the same region of $c$ that Section~\ref{sec:exp-anneal} could cross only when it was inherited from a stable checkpoint. The rectified-forward operator, whose full gradient is the more principled design, still does not train.
\paragraph{The cell we did not run.}
\label{sec:family-gap}
Over $\{D^+,\,D^++D^-\}$ in each pass, the family covers three of four cells; the matched-rectified operator ($D_f=D_b=D^+$) is absent. Every failing member above therefore pairs a small $D_f$ with a $D^-$-inclusive backward, so forward degeneracy and forward/backward mismatch remain entangled: we cannot rule out that the rectified-forward failure is driven by the mismatch rather than by the degeneracy alone. Corollary~\ref{cor:discontinuity} predicts the matched case diverges too, since its gradient is $\Theta((\eta+\epsilon)^{-1})$ on near-dead rows regardless of $D_b$, and Lemma~\ref{lem:dead-row} rules out revival by any backward choice on exactly-dead rows. Neither argument substitutes for the run. Note also that the rectified-only variant of \citet{zuhri2025softpick} is not this cell: their Appendix~B form masks negative scores to $-\infty$ and applies an ordinary softmax, which normalizes over the positive subset and never produces a near-zero denominator, which is why it trains where our rectified-forward operator does not. This is the first item of Appendix~\ref{app:future}.

\section{Supplementary figures}
\label{app:supp-figures}

Figure~\ref{fig:attn-surface-layer2} is the layer-2 companion to Figure~\ref{fig:attn-surface}: softpick keeps one dominant spike, and the $D^+$-only forward again scatters it into several spurious peaks, at the same maximum weight (0.1875) at this layer, unlike the layer-8 case where the $D^+$-only peak exceeds softpick's own. Figure~\ref{fig:anneal-full} is the full four-panel training curve, loss, gradient norm, learning rate, and the $c$-schedule, for the $c{:}1\to0$ anneal of Section~\ref{sec:exp-anneal}.

\begin{figure}[h]
\centering
\includegraphics[width=\linewidth]{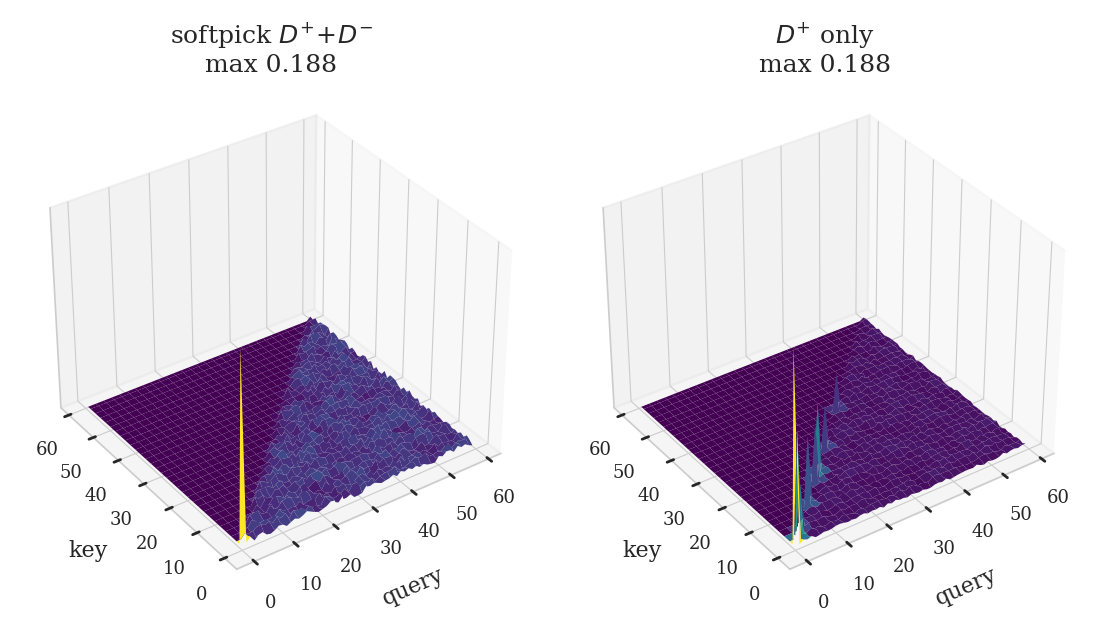}
\caption{Attention-weight surface at layer 2, companion to Figure~\ref{fig:attn-surface}. Softpick and the $D^+$-only forward reach the same maximum weight, 0.1875, but the $D^+$-only forward redistributes it across several spurious peaks.}
\label{fig:attn-surface-layer2}
\end{figure}

\begin{figure}[h]
\centering
\includegraphics[width=\linewidth]{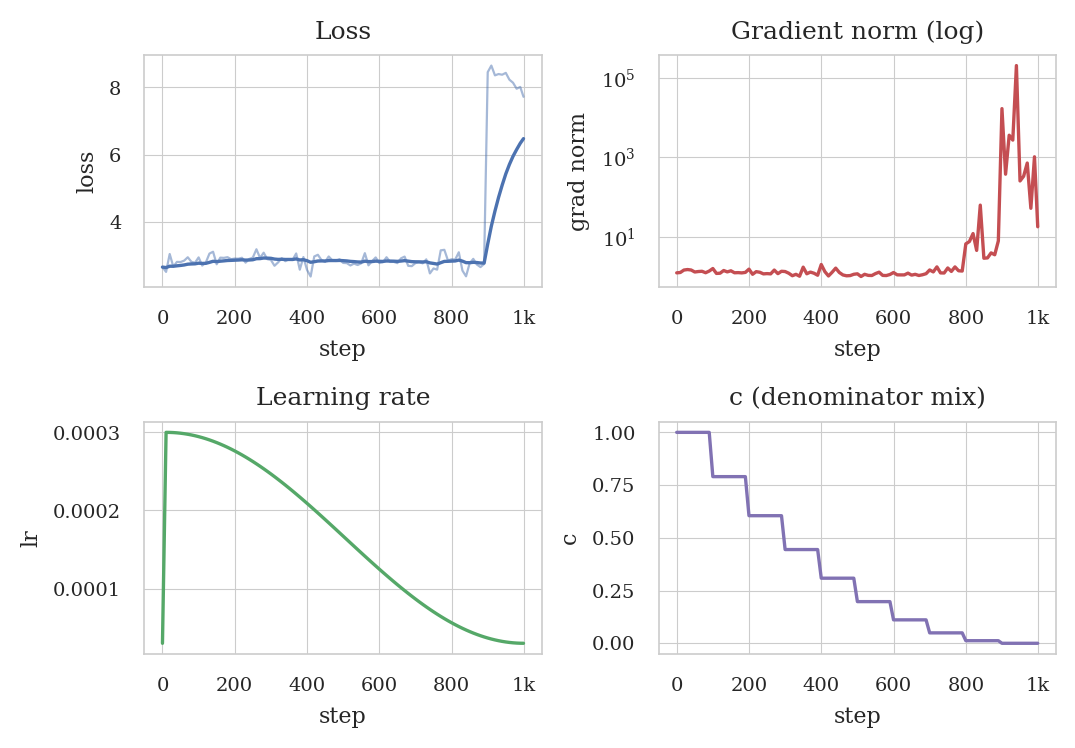}
\caption{Full training panel for the $c{:}1\to0$ anneal at 340M: loss, gradient norm (log scale), learning rate, and the $c$-schedule, all against step.}
\label{fig:anneal-full}
\end{figure}

\section{Training configuration}
\label{app:config}

\begin{table}[h]
\centering
\small
\caption{Training configuration, read from our run configurations and training scripts and the publicly released 340M checkpoint.}
\label{tab:config}
\resizebox{\linewidth}{!}{%
\begin{tabular}{lccc}
\toprule
& 130M & 230M & 340M \\
\midrule
Parameters & 130M & 228.9M & 340M \\
Hidden size & 768 & 896 & 1024 \\
Layers & 12 & 17 & 24 \\
Attention heads (total) & 144 (12/layer) & 238 (14/layer) & 384 (16/layer) \\
Head dim & 64 & 64 & 64 \\
MLP & SwiGLU, ratio 4 & SwiGLU, ratio 4 & SwiGLU, ratio 4 \\
Vocabulary & 32{,}000 & 32{,}000 & 32{,}000 \\
Peak learning rate & $4\times10^{-4}$ & $3.5\times10^{-4}$ & $3\times10^{-4}$ (anneal continuation) \\
LR schedule & cosine, 300-step warmup, min 10\% & cosine, 300-step warmup, min 10\% & cosine, 1\% warmup, min 10\% \\
Max grad norm & 1.0 & 1.0 & 1.0 \\
Precision & bf16 & bf16 & fp32 (anneal), bf16 (base ckpt) \\
Batch (local $\times$ accum $\times$ GPUs) & $4\times16\times2=128$ & $4\times16\times2=128$ & $2\times4=8$ (anneal) \\
GPUs & 2$\times$ A6000 & 2$\times$ A6000 & 1$\times$ A6000 (anneal) \\
Position encoding & RoPE, $\theta=10{,}000$ & RoPE, $\theta=10{,}000$ & RoPE, $\theta=10{,}000$ \\
Max context & 8192 & 8192 & 8192 (2048 anneal) \\
Kernel & fast Triton & fast Triton & eager (anneal), Triton (base) \\
Data & FineWeb-Edu & FineWeb-Edu & FineWeb-Edu, \texttt{sample-100BT} \\
Softpick $\epsilon$ & $10^{-6}$ & $10^{-6}$ & $10^{-8}$ (anneal continuation) \\
Steps trained & 20k+ & 30k & 1000-step anneal \\
\bottomrule
\end{tabular}}
\end{table}

\paragraph{Compute.} Each from-scratch grid run (130M and 230M, one per operator) uses two NVIDIA RTX A6000 GPUs (48GB each) in data-parallel (DDP world size 2), for a global batch of 128 sequences of length 2048, matching the softpick paper's global batch; a single 30k-step 230M run takes roughly three days on this setup. The 340M anneal continuation runs on a single A6000, using the eager, non-Triton softpick path needed for a per-step-modifiable backward, with AdamW ($\epsilon=10^{-15}$), batch size 2, gradient accumulation 4, and gradient checkpointing; its sequence length drops from the base checkpoint's 4096 to 2048 because materializing the full attention matrix at 4096 does not fit alongside the model and optimizer state on 48GB. The software stack is \texttt{torch} 2.5.1+cu124, \texttt{triton} 3.1.0, \texttt{flash-linear-attention} 0.1.2, and \texttt{transformers} 4.51.3, with AdamW \citep{loshchilov2019decoupled} throughout and FineWeb-Edu \citep{penedo2024fineweb} as the only training data.

\section{Evaluation protocols}
\label{app:eval}

\paragraph{Passkey retrieval.} We use the landmark task \citep{mohtashami2023landmark} (harness in Appendix~\ref{app:implementation}, scoring rules in Appendix~\ref{app:scorer}). Each prompt follows a fixed template: a task description (``There is an important info hidden inside a lot of irrelevant text. Find it and memorize them.''), then \texttt{n\_garbage}$=\lceil 3.6\times\text{ctx\_len}\rceil$ characters of a repeated filler sentence (``The grass is green. The sky is blue. \dots''), with an information line (``The pass key is $P$. Remember it. $P$ is the pass key.'') inserted at a uniformly random offset inside the filler, closing with the query ``What is the pass key? The pass key is''. The passkey $P$ is an integer drawn uniformly from $[1,50000)$. We evaluate at context lengths $\sim$1024, $\sim$2048, and $\sim$3200 with greedy decoding, using $n=30$ seeds for the inference sweeps and $n=500$ for the trained-model comparisons.

\paragraph{Kurtosis.} We report excess (Fisher) kurtosis. Hidden-state kurtosis is computed over the flattened residual-stream activations of each layer and averaged across layers, with the per-layer range noted where relevant; attention-weight kurtosis is computed over the post-normalization attention matrix, reported as the layer mean at 130M and as the single worst layer's peak at 230M. Higher kurtosis indicates the heavier-tailed, outlier-prone activations that softpick was introduced to suppress.

\paragraph{Low-bit quantization.} We quantize the model weights with round-to-nearest (RTN) at the stated bit-width and report the loss delta, the increase in fp32-evaluation cross-entropy caused by quantization, so that $0$ is lossless and larger is worse. We report 2-bit and 4-bit only (Appendix~\ref{app:scorer}).

\paragraph{Dead heads.} A head is counted dead at threshold $\tau$ if its maximum attention-output magnitude stays below $\tau$ for more than 95\% of tokens. We report $\tau=10^{-6}$ throughout, and $\tau=10^{-2}$ where noted; sensitivity to $\tau$ is discussed in Appendix~\ref{app:corrections}.

\section{The 340M annealed checkpoint}
\label{app:corrections}

The step-700 annealed checkpoint of Section~\ref{sec:exp-anneal} ($c=0.111$, the last stable step before collapse) reaches fp32 loss 2.2173. Its mean hidden-state kurtosis is $\approx$100 (range 32 to 282), the lowest we measure anywhere in this paper, while its mean attention-weight kurtosis across the 24 layers is 209{,}718, with individual layers far higher (layer 1: 2{,}758{,}599). Its round-to-nearest quantization deltas are $+6.65$ at 2-bit and $+4.85$ at 4-bit. So the interior operating point that survives has the best activation statistics on the hidden-state axis and yet is structurally degenerate ($D^+$ 5th percentile $\approx0$ in all 24 layers, 61 dead heads), which is the central tension of the annealing result: the healthiest-looking activations sit at an operating point that cannot be trained into existence from scratch.

\paragraph{Dead-head threshold sensitivity.}


The dead-head criterion used throughout this paper, and in the source we build on, is a threshold on how often a head's output falls below a small constant. That threshold matters more than a fixed property of the model would suggest. At the 340M annealed checkpoint of Section~\ref{sec:exp-anneal}, the same model has 61 of 384 heads (15.9\%) dead under a threshold of $10^{-6}$ and 91 of 384 (23.7\%) dead under $10^{-2}$, a 49\% relative increase from a two-order-of-magnitude change in the cutoff. A head that crosses only one of these thresholds is dormant under the stricter definition and dead under the looser one, and the fact that the count moves this much with the threshold is itself evidence that the metric tracks continuous, threshold-sensitive dormancy rather than binary, permanent death. We would extend the same caution to any cross-scale comparison of dead-head counts that fixes a single threshold and treats the resulting count as if it measured the same quantity at every model size.

\section{Future work}
\label{app:future}

Every comparison in this paper rests on a single checkpoint per operator at 130M and 230M, with no seeds and no error bars, so the closing gaps we report are two points consistent with a trend rather than a demonstrated one. The natural next step is a training ladder with seeds at a larger scale, such as 450M, holding tokens-per-parameter and data fixed and staying within RoPE range, read against the scaling-law framing of \citet{kaplan2020scaling} and the compute-optimal token budgets of \citet{hoffmann2022training}. Three quantities would need to hold for the stop-gradient operator to become a recommendation rather than a survivor: the hidden-state and quantization gap-versus-parameter-count slope between 130M and 230M continuing to close, the dead-head slope continuing in the stop-gradient operator's favor, since that story is entirely about how counts grow with scale, and the peak attention-weight kurtosis residual either closing or proving to matter less than it currently appears to. A de-risking rung at an intermediate scale, run before committing to the full ladder, would test whether the amortization trend survives seeds at all. The digit-level passkey analysis of Appendix~\ref{app:scorer}, which we ran here only for softpick's 230M failure, should be extended across every operator and scale so retrieval is reported by digit recall and edit distance throughout rather than by exact match alone.

\section{Ethics statement}
\label{app:ethics}

This work trains and evaluates language models only on public data: FineWeb-Edu \citep{penedo2024fineweb}, a filtered subset of the public Common Crawl, together with the publicly released softpick 340M checkpoint \citep{zuhri2025softpick} for the inference and annealing experiments. No private, proprietary, or personally identifying data is used, and no human subjects are involved. The models are small (130M to 340M parameters) and trained for short budgets, so the compute footprint is modest, on the order of a few GPU-days per run on two RTX A6000s. The contribution is a diagnostic analysis of an existing attention operator and introduces no new capability or data, so it carries no greater dual-use risk than the base architecture it studies.

\section*{NeurIPS Paper Checklist}

\begin{enumerate}

\item {\bf Claims}
    \item[] Question: Do the main claims made in the abstract and introduction accurately reflect the paper's contributions and scope?
    \item[] Answer: \answerYes{}
    \item[] Justification: The abstract and introduction (Section~\ref{sec:introduction}) state the three contributions---the operator family over (forward, backward) denominators, the localization of the rectified-forward failure at initialization, and the 130M/230M comparison of the stop-gradient operator against softpick---and each is substantiated in Sections~\ref{sec:problem-formulation}--\ref{sec:experiments} and the appendix.

\item {\bf Limitations}
    \item[] Question: Does the paper discuss the limitations of the work performed by the authors?
    \item[] Answer: \answerYes{}
    \item[] Justification: Limitations are discussed throughout: every comparison rests on a single checkpoint per operator with no seeds or error bars, the stop-gradient operator's biased-gradient status is stated explicitly (Appendix~\ref{app:stability}), and the future-work section (Appendix~\ref{app:future}) lays out what would need to hold for the operator to become a recommendation rather than a survivor.

\item {\bf Theory assumptions and proofs}
    \item[] Question: For each theoretical result, does the paper provide the full set of assumptions and a complete (and correct) proof?
    \item[] Answer: \answerYes{}
    \item[] Justification: The theoretical results are numbered and stated with their assumptions: Lemma~\ref{lem:dead-row} (dead row) has a full proof in Appendix~\ref{app:zerosupport}, Corollary~\ref{cor:discontinuity} and Lemma~\ref{lem:bounded} are derived there, and Propositions~\ref{prop:zero-support} and~\ref{prop:count} state their idealizing assumptions (symmetric, independent initial scores) explicitly and flag where they hold only approximately.

    \item {\bf Experimental result reproducibility}
    \item[] Question: Does the paper fully disclose all the information needed to reproduce the main experimental results of the paper to the extent that it affects the main claims and/or conclusions of the paper (regardless of whether the code and data are provided or not)?
    \item[] Answer: \answerYes{}
    \item[] Justification: Appendix~\ref{app:implementation} documents the straight-through decoupling construction and kernel changes for every operator, Appendix~\ref{app:config} gives the full training configuration and compute (architecture, learning-rate schedules, batch, precision, software versions, $\epsilon$ values), and Appendices~\ref{app:eval} and~\ref{app:scorer} specify the passkey harness, kurtosis, quantization, and dead-head protocols; the base model is the publicly released softpick 340M checkpoint and the data is public FineWeb-Edu.

\item {\bf Open access to data and code}
    \item[] Question: Does the paper provide open access to the data and code, with sufficient instructions to faithfully reproduce the main experimental results, as described in supplemental material?
    \item[] Answer: \answerNo{}
    \item[] Justification: The training data (FineWeb-Edu) and the base 340M checkpoint are publicly available and cited, but the code for the operator family is not released with this anonymized submission. The implementation is documented in enough detail to reproduce (Appendix~\ref{app:implementation}), including the exact kernel change that defines each operator, and we intend to release code with the final version.

\item {\bf Experimental setting/details}
    \item[] Question: Does the paper specify all the training and test details (e.g., data splits, hyperparameters, how they were chosen, type of optimizer) necessary to understand the results?
    \item[] Answer: \answerYes{}
    \item[] Justification: Appendix~\ref{app:config} tabulates the architecture, optimizer (AdamW), learning-rate schedules and warmup, batch composition, precision, gradient clipping, context length, and step counts at each scale, and Appendix~\ref{app:eval} specifies the evaluation protocols and decoding settings.

\item {\bf Experiment statistical significance}
    \item[] Question: Does the paper report error bars suitably and correctly defined or other appropriate information about the statistical significance of the experiments?
    \item[] Answer: \answerNo{}
    \item[] Justification: Each kurtosis/quantization/dead-head comparison rests on a single checkpoint per operator per scale with no seeds, given the compute cost (roughly three GPU-days per 230M run); the paper states this limitation explicitly (Appendix~\ref{app:full-cost}, \ref{app:future}) and treats the closing gaps as two points consistent with a trend rather than a demonstrated effect. Passkey retrieval is reported over $n=500$ samples per cell with two seed offsets checked for softpick's 230M failure (Appendix~\ref{app:full-cost}).

\item {\bf Experiments compute resources}
    \item[] Question: For each experiment, does the paper provide sufficient information on the computer resources (type of compute workers, memory, time of execution) needed to reproduce the experiments?
    \item[] Answer: \answerYes{}
    \item[] Justification: The compute paragraph in Appendix~\ref{app:config} states the hardware (two NVIDIA RTX A6000 48GB GPUs in DDP for the grid runs, a single A6000 for the anneal), the wall-clock cost (a 30k-step 230M run takes roughly three days), and the memory-driven choices (sequence length dropped to 2048 to fit on 48GB).
    
\item {\bf Code of ethics}
    \item[] Question: Does the research conducted in the paper conform, in every respect, with the NeurIPS Code of Ethics \url{https://neurips.cc/public/EthicsGuidelines}?
    \item[] Answer: \answerYes{}
    \item[] Justification: The work uses only public data and a publicly released checkpoint, involves no human subjects or private data, and preserves anonymity (Appendix~\ref{app:ethics}); it conforms with the NeurIPS Code of Ethics.

\item {\bf Broader impacts}
    \item[] Question: Does the paper discuss both potential positive societal impacts and negative societal impacts of the work performed?
    \item[] Answer: \answerYes{}
    \item[] Justification: The ethics statement (Appendix~\ref{app:ethics}) notes that the contribution is a diagnostic analysis of an existing attention operator that introduces no new capability or data, so it carries no greater dual-use risk than the base architecture; the positive impact is toward more quantization-friendly, efficient attention.
    
\item {\bf Safeguards}
    \item[] Question: Does the paper describe safeguards that have been put in place for responsible release of data or models that have a high risk for misuse (e.g., pre-trained language models, image generators, or scraped datasets)?
    \item[] Answer: \answerNA{}
    \item[] Justification: The paper releases no data or models; it studies small (130M--340M) research models on public data, which pose no high risk for misuse (Appendix~\ref{app:ethics}).

\item {\bf Licenses for existing assets}
    \item[] Question: Are the creators or original owners of assets (e.g., code, data, models), used in the paper, properly credited and are the license and terms of use explicitly mentioned and properly respected?
    \item[] Answer: \answerYes{}
    \item[] Justification: All existing assets are cited to their original work: the softpick 340M checkpoint~\citep{zuhri2025softpick}, the FineWeb-Edu dataset~\citep{penedo2024fineweb}, and the software stack (flash-linear-attention, PyTorch, Triton, Transformers) in Appendix~\ref{app:config}, and they are used consistently with their open-source terms.

\item {\bf New assets}
    \item[] Question: Are new assets introduced in the paper well documented and is the documentation provided alongside the assets?
    \item[] Answer: \answerNA{}
    \item[] Justification: The paper does not release new datasets or model checkpoints; its contribution is a diagnostic analysis and a family of attention operators described in the text (Section~\ref{sec:problem-formulation}, Appendix~\ref{app:implementation}).

\item {\bf Crowdsourcing and research with human subjects}
    \item[] Question: For crowdsourcing experiments and research with human subjects, does the paper include the full text of instructions given to participants and screenshots, if applicable, as well as details about compensation (if any)?
    \item[] Answer: \answerNA{}
    \item[] Justification: The paper involves no crowdsourcing and no research with human subjects (Appendix~\ref{app:ethics}).

\item {\bf Institutional review board (IRB) approvals or equivalent for research with human subjects}
    \item[] Question: Does the paper describe potential risks incurred by study participants, whether such risks were disclosed to the subjects, and whether Institutional Review Board (IRB) approvals (or an equivalent approval/review based on the requirements of your country or institution) were obtained?
    \item[] Answer: \answerNA{}
    \item[] Justification: The paper involves no research with human subjects, so IRB approval is not applicable (Appendix~\ref{app:ethics}).

\item {\bf Declaration of LLM usage}
    \item[] Question: Does the paper describe the usage of LLMs if it is an important, original, or non-standard component of the core methods in this research? Note that if the LLM is used only for writing, editing, or formatting purposes and does \emph{not} impact the core methodology, scientific rigor, or originality of the research, declaration is not required.
    \item[] Answer: \answerNA{}
    \item[] Justification: LLMs were used only for grammar fixes and refining the writing. They played no role in the core methodology, scientific rigor, or originality of the research, so a declaration is not required.

\end{enumerate}

\end{document}